\documentclass[final]{colt2022} 
\usepackage{amsmath}
\usepackage{enumitem}
\usepackage{hyperref}
\usepackage[capitalize, noabbrev]{cleveref}
\usepackage{preamble}
\usepackage{array}
\usepackage{graphicx}
\usepackage[skip=3pt]{caption}

\title[Trace norm regularization for multi-task learning]{Trace norm regularization for multi-task learning with scarce data}
\usepackage{times}
\coltauthor{%
 \Name{Etienne Boursier} \Email{etienne.boursier@epfl.ch}\\
 \addr TML Lab, EPFL, Switzerland
 \AND
 \Name{Mikhail Konobeev} \Email{mikhail.konobeev@epfl.ch}\\
 \addr TML Lab, EPFL, Switzerland
 \AND
  \Name{Nicolas Flammarion} \Email{nicolas.flammarion@epfl.ch}\\
 \addr TML Lab, EPFL, Switzerland
}

\begin{document}
\raggedbottom
\setlength{\abovedisplayskip}{3pt}
\maketitle

\begin{abstract}%
  Multi-task learning leverages structural similarities between multiple tasks to learn despite very few samples. Motivated by the recent success of neural networks applied to data-scarce tasks, we consider a linear low-dimensional shared representation model.
  Despite an extensive literature, existing theoretical results either guarantee weak estimation rates or require a large number of samples per task.
  This work provides the first estimation error bound for the trace norm regularized estimator when the number of samples per task is small. 
  The advantages of trace norm regularization for learning data-scarce tasks extend to meta-learning and are confirmed empirically on synthetic datasets.
\end{abstract}


\section{Introduction}

Common supervised learning requires a large number of training examples, which are often costly in time and resources to acquire. The available dataset for a single task can be very limited, making impossible to learn solely based on it.
Multi-task learning instead estimates a model across multiple tasks, by leveraging structural similarities among them. It jointly uses all datasets and thus learns efficiently as already observed in numerous applications including natural language processing \citep{ando2005framework}, image segmentation \citep{cheng2011multi} and medical prediction \citep{caruana1997multitask}.

\medskip

This work considers the problem of multi-task learning, where an unknown linear low-dimensional representation is shared among different tasks \citep{rohde2011estimation}. It studies the following question: \textit{how can we learn across multiple tasks with a very limited number of observations for each of them?}
This question is also of fundamental interest to meta-learning and few-shot learning, which aim at aggregating knowledge among multiple tasks to learn a shared representation \citep{vinyals2016matching,finn2017model}.

In spite of the vast multi-task learning literature, the existing results remain unsatisfying. In particular, guarantees on trace norm regularization \citep{rohde2011estimation} and Burer-Monteiro factorization \citep{tripuraneni2021provable} both assume that the number $m$ of observations per task is large. The former assumes it is larger than the features dimension $d$, while the latter assumes it scales logarithmically in $T$, the total number of tasks. Such conditions are not always met in practice, when it is much easier to acquire high dimensional data on new tasks than on existing ones \citep[see e.g.][]{wang2017learning}.
On the other hand, the Method of Moments \citep{tripuraneni2021provable} learns with a very limited number of observations per task, but requires very specific feature distributions.

Similarly to \citet{rohde2011estimation}, we study the trace norm regularized estimator. It is a natural choice when estimating low rank matrices, since the trace norm convexifies the rank function. Trace norm based methods have already been successfully used in numerous multi-task learning applications \citep{amit2007uncovering,cheng2011multi,harchaoui2012large}, but lack theoretical guarantees when the number of observations per task is limited.

\paragraph{Contributions.} This work bounds the estimation error of the trace norm regularized estimator with a few observations per task $(m<d)$. 
The analysis becomes particularly intricate when the number of samples per task is smaller than the features dimension, since no restricted isometry condition holds \citep{rohde2011estimation}. Instead, our analysis uses a weaker restricted strong convexity condition \citep{van2009conditions}. 
Proving that this restricted strong convexity condition holds is our main technical contribution, besides upper bounding a stochastic term using concentration of heavy tailed distributions. These techniques lead to our main result, of which an informal version is given in \cref{thm:informal}.
\begin{thm}[informal]\label{thm:informal}
For any number of observations per task $m$, the trace norm regularized estimator $\hat{M}$ satisfies with high probability
\begin{equation*}
    \|\hat{M} - M^* \|_F \leq \tbigO{\sigma\sqrt{r\frac{\frac{d^2}{m}+T}{m}} + \sqrt{rd\frac{d+T}{m^2}}},
\end{equation*}
where $T$ is the number of tasks, $d$ is the dimension of the feature space, $\sigma^2$ is the variance of the label noise, $M^* \in \R^{d\times T}$ is the ground-truth parameter matrix and $r$ is its rank. The notation $\tilde{\mathcal{O}}$ hides multiplicative constants and logarithmic terms in $d,m$ and $T$.
\end{thm}
Note that by using linear regressions on each individual task independently, the estimation error scales as $\bigO{\sigma\sqrt{\frac{dT}{m}}}$ \citep{hsu2012random}. In the regime when the number of tasks is large, trace norm regularization thus improves this estimation by a factor $\sqrt{\frac{r}{m}}$, leveraging the low rank structure of the parameter matrix. 
A gap yet remains with the oracle baseline knowing beforehand the $r$-dimensional subspace induced by the parameters. This baseline computes linear regressions with $r$ parameters and thus has an error scaling as $\bigO{\sigma\sqrt{\frac{rT}{m}}}$.

To our knowledge, \cref{thm:informal} is the first general estimation error bound for a multi-task estimator with an arbitrarily small number of observations per task. As discussed in \cref{sec:related}, a better bound can be proven for the Method of Moments in this setting\footnote{Such a bound is proven in \cref{sec:MoM} and is a minor contribution of this work.}, but it only holds for a very specific data model (\eg Gaussian) and behaves much worse in practice as highlighted in \cref{sec:simulations}.

\cref{thm:informal} also allows in \cref{sec:meta} to bound the estimation error for a new, previously unobserved task. This result illustrates the interest of trace norm regularization for meta-learning. Finally, we compare empirically different multi-task regression methods and discuss the practical advantages/drawbacks of trace norm regularization in \cref{sec:simulations}. 

\section{Model}\label{sec:model} 

\paragraph{Notations.} In the following, $[n] \coloneqq \lbrace 1, \ldots, n \rbrace$. For a matrix $M \in \R^{d\times T}$, $M^{(t)}\in\R^d$ denotes its $t$-th column, $\lambda_i(M)$ its $i$-th largest singular value and $\|M\|_*$ its trace (or nuclear) norm, \ie $\|M\|_* = \sum_{i=1}^{\min(d,T)} \lambda_i(M)$ . We use the notation $\langle\cdot,\cdot\rangle$ for the canonical inner product both for vectors and matrices.

\medskip

\paragraph{Model.} In the remaining of the paper, we consider the model described in this section. There are $T$ tasks, each of which contains $m$ observation samples $(x_i^t,y_i^t)\in\R^d \times \R$. 
We consider the linear model
\begin{equation}\label{eq:linear-mode}
    y_i^t = \langle M^{^*\,(t)}, x_i^t\rangle + \eps_i^t \quad \text{ for any } (i,t)\in[m]\times[T],
\end{equation}
where $M^*$ is the matrix of parameters to estimate.
%
We assume in the following that $\mathrm{rank}(M^*) = r$, where $r\ll d$, and that the features and noise variables are well behaved as stated in \cref{ass:random}.
\begin{assumption}[Random design]\label{ass:random}
The $(x_i^t)$ are independent centered $1$-sub-Gaussian random variables and the $\eps_i^t$ are independent centered $\sigma$-sub-Gaussian random variables.
Moreover, the features are isotropic, \ie $\E[(x_i^t)^\top x_i^t] = I_d$.
\end{assumption}
We also assume the task diversity condition, which claims that the scale of the parameters is roughly the same for all tasks.
\begin{assumption}[Task diversity]\label{ass:diversity}
Given some constant $C$, the parameters matrix $M^*$ verifies:
\begin{equation*}
\max_{t \in [T]} \|M^{*\, (t)}\|^2 \leq C.
\end{equation*}
\end{assumption}
The task diversity assumption has been introduced by \citet{tripuraneni2021provable} and is also considered in subsequent works  \citep{thekumparampil2021sample,thekumparampil2021statistically}. It ensures that a single task does not get too significant with respect to the others. However, we do not require any lower bound on the norm of task parameters. 

\section{Related work}\label{sec:related}

This section discusses the related literature and \cref{table:multilearning} summarizes the available error bounds for the model described in \cref{sec:model}.

Different structural assumptions have been considered in the multi-task literature. For example, \citet{denevi2019learning,cesa2021multitask} assume that the task parameters all lie in a small Euclidean ball and \citet{argyriou2008convex,lounici2009taking} assume that each parameter vector is sparse and its support is shared among the tasks. In the latter, the parameter matrix $M^*$ has a small $\ell_{2,1}$ norm. This paper studies a classical structural assumption generalizing the sparse setting: the parameter matrix has a small rank.
In that case, it seems natural to consider the following estimator
\begin{equation}\label{eq:rankopt}
    \argmin_{\substack{M \in \R^{d\times T}\\\mathrm{rank}(M) \leq r}} \frac{1}{mT}\sum_{(i,t)\in[m]\times[T]}\pr{y_i^t - \langle M^{(t)}, x_i^t \rangle}^2.
\end{equation}
When the features are shared among the tasks, \ie $x_i^t= x_i^{t'}$, the considered model is equivalent to multivariate regression \citep{izenman1975reduced,obozinski2008union} and a closed form solution of \cref{eq:rankopt} is known \citep{bunea2011optimal}. \citet{maurer2016benefit} bound the error of this estimator in a general multi-task setting, using Gaussian complexity arguments. Besides holding only for bounded Lipschitz loss functions, this bound is weaker than what can be obtained for the squared loss, since it does not use any smoothness property on the loss function.

\medskip

Computing the above optimization program yet becomes intractable when the features differ among the tasks, which corresponds to the setting of interest. 
A natural approach replaces the rank constraint by a trace norm constraint, since it convexifies the rank function (similarly to the $\ell_1$ norm that convexifies the $\ell_0$ norm). Equivalently, a regularized problem can be considered:
\begin{equation*}
    \argmin_{M \in \R^{d\times T}} \frac{1}{mT}\sum_{(i,t)\in[m]\times[T]}\pr{y_i^t - \langle M^{(t)}, x_i^t \rangle}^2 + \lambda \|M\|_*.
\end{equation*}
Multi-task learning can be considered as a particular case of matrix completion. \citet{candes2011tight,rohde2011estimation} studied low-rank matrix completion, using restricted isometry conditions. In particular, \citet{rohde2011estimation} bound the error of the trace norm regularized estimator described above for multi-task learning. However, they assume a restricted isometry condition, which only holds when the number of samples per task is larger than the features dimension $(m\geq d)$, limiting the interest of this result in practice.

The only known error bounds for trace norm based methods when the number of observations per task is small ($m<d$) derive from Rademacher complexity arguments \citep{pontil2013excess,yousefi2018local}. For the same reasons as \citet{maurer2016benefit}, they only hold for Lipschitz loss functions and are weaker than what can be obtained for the squared loss.

\medskip

Other approaches yet manage to provide near tight error bounds when the number of observations per task is smaller than the features dimension. In particular, the Burer-Monteiro factorization considers the problem
\begin{equation*}
        \argmin_{\substack{U \in \R^{d\times r}\\V\in\R^{T\times r}}} \frac{1}{mT}\sum_{(i,t)\in[m]\times[T]}\pr{y_i^t - \langle (UV\tp)^{(t)}, x_i^t \rangle}^2.
\end{equation*}
In words, low-rank matrices are factorized as $M=UV\tp$. This optimization problem is equivalent to \cref{eq:rankopt}. It is not convex in its arguments $(U,V)$, but only bilinear. As a consequence, we can only aim at computing a local minimum of the objective, for example with first order optimization methods.
\citet{tripuraneni2021provable} nevertheless bounded the estimation error of any local minimum of the above optimization program. Their work is the closest in spirit to ours and is motivated by the theoretical study of meta-learning.

\citet{thekumparampil2021sample,thekumparampil2021statistically} recently improved the error guarantees of the Burer-Monteiro factorization, in terms of the distance between the estimated $r$-dimensional features subspace and the ground-truth one. It is achieved using an alternate minimization algorithm.
This allows to provide tighter estimation bounds on a new task in the meta-learning setting considered in \cref{sec:meta}, but does not improve the existing multi-task bounds. 
However, all the bounds for Burer-Monteiro factorization require that the number of samples per task scales with $\log(T)$. Since we might consider a very large number of tasks in practice, along with a limited number of samples per task, this requirement is a major drawback.

\medskip

The Method of Moments introduced by \citet{tripuraneni2021provable} is actually the only estimator that provides satisfying bounds with a very limited number of observations per task. It directly estimates the $r$-dimensional features subspace. Yet, only a bound on the error of the subspace estimation is known, besides requiring the feature distribution to be Gaussian. In \cref{sec:MoM}, we extend these results to an error bound on the whole estimated parameters matrix and to any spherically symmetric feature distribution.
This assumption on the feature distribution is yet often unverified in practice, and the Method of Moments might fail to learn the features subspace without it as shown in \cref{sec:simulations}. Moreover, it empirically performs poorly with respect to the other estimators even for Gaussian distributions, as observed in \cref{sec:simulations}.

\medskip

Multi-task classification has also been studied in previous works \citep{maurer2006bounds,cavallanti2010linear}, but is not further discussed as it is beyond the scope of this paper.

\begin{table}[t]
{{\setlength{\extrarowheight}{5pt}
\begin{adjustwidth}{-5cm}{-5cm}
\centering
\begin{tabular}{|c|c|c|c|}
\hline
\textbf{Estimator}  & \textbf{Error bound $\|\hat{M}-M^*\|_F$} 
& \textbf{Samples per task} & \textbf{Extra assumption} \\[5pt]
\hline
\begin{tabular}{@{}c@{}}Trace norm regularization \\[-3pt] \citep{rohde2011estimation}\end{tabular} & $\sigma\sqrt{r\frac{d+T}{m}}$ & $\Omega(d)$ & Deterministic features \\\hline
\begin{tabular}{@{}c@{}}Burer-Monteiro factorization \\[-3pt] \citep{tripuraneni2021provable}\end{tabular} & $\sigma\sqrt{r\frac{d+T}{m}}$  & $\Omega(r^4\log(T))$ & - \\\hline
\begin{tabular}{@{}c@{}}Method of Moments \\[-3pt] \cref{thm:MoM} adapted from\\[-3pt] \citep{tripuraneni2021provable}\end{tabular} & $\sigma\sqrt{r\frac{\sigma^2rd+T}{m}} + r\sqrt{\frac{d}{m}}$  & $\Omega(r\log(r))$ & \begin{tabular}{@{}c@{}}Spherically symmetric \\[-3pt] feature distribution\end{tabular}  \\\hline
\begin{tabular}{@{}c@{}}\textbf{Trace norm regularization} \\[-3pt] \textbf{\cref{thm:estimation1}}\end{tabular} & $\sigma\sqrt{r\frac{\frac{d^2}{m}+T}{m}} + \sqrt{rd\frac{d+T}{m^2}}$ & $\Omega(1)$ & \textbf{-} \\\hline
\end{tabular}
\end{adjustwidth}
\caption{\label{table:multilearning} Different bounds for multi-task learning. Only the dependencies in $\sigma,r,d,m,T$ are provided and eventual logarithmic terms are omitted. Our main result is highlighted in bold.}}}
\end{table}

\section{Bound on the estimation error}\label{sec:bound}

In this section, we provide error guarantees for the estimator
\begin{equation}\label{eq:nuc-opt-prob}
\hat{M} = \argmin_{M \in \mathcal{W}} \frac{1}{mT} \sum_{(i,t)\in[m]\times[T]} (y_i^t-\langle x_i^t, M^{(t)} \rangle)^2 + \lambda \| M\|_{*}
\end{equation}
where $\mathcal{W} = \left\lbrace M \in \R^{d\times T} \mid \max_{t \in [T]} \|M^{(t)}\|^2 \leq C \right\rbrace$ and $C$ is the constant introduced in \cref{ass:diversity}.
We restrict the estimator to the ball $\mathcal{W}$ for analysis purpose, but we do not need to enforce this constraint in practice, i.e. we empirically obtain good results when solving the unconstrained problem. 

\medskip

Our proof relies on the decomposability of the trace norm \citep{negahban2012unified}. 
Since the restricted isometry condition does not hold with a limited number of observations per task, we instead prove a restricted strong convexity condition. Its proof is particularly difficult, since the condition is non-uniform and considered for an intricate subset of matrices. 
On the other hand, bounding the effective noise level is challenging because of the random design model. As a consequence, our analysis uses concentration on heavy tailed distributions, while previous works on trace norm regularization use Bernstein inequalities that only hold for sub-exponential distributions.

Showing both restricted strong convexity and noise level conditions is the main technical challenge of this work.
These conditions are respectively presented in \cref{sec:rsc,sec:noise-level}.
We now state our main result, which bounds the error of the estimator defined in \cref{eq:nuc-opt-prob}. Its proof is given in \cref{sec:mainproof}.
\begin{thm}\label{thm:estimation1}
Assume $T = \Omega(d)$, for $\lambda = 4\tau$ where $\tau = \frac{c_2\sigma}{\sqrt{T}}
\sqrt{\frac{T + \nicefrac{d^2}{m}}{mT}}$, with probability at least $1- (2T+c_0)e^{-c_1 d} - 2e^{-c_1r(d+T)}$:
\begin{equation}\label{eq:main}
    \|\hat{M} - M^* \|_F \leq c \sigma\sqrt{r\frac{\frac{d^2}{m}+T}{m}} + c \sqrt{Crd\frac{d+T}{m^2}\ln\pr{\frac{dT}{m}}},
\end{equation}
where $c, c_0$, $c_1$ and $c_2$ are universal positive constants.
\end{thm}
%
This bound is of the same order as the known error bound when the number of samples per task is larger than the dimension \citep{rohde2011estimation}. This result is of great significance when $m<\min\pr{d,\log(T)}$ since it provides the first estimation guarantees in this case. We recall it is the regime of interest in most applications.
It illustrates the success of trace-norm methods in multi-task learning settings. It indeed leads to a $\sqrt{\frac{r}{m}}$ estimation improvement with respect to the single-task baseline, which proceeds to $T$ independent linear regressions.

\medskip

We believe that the extra $\sqrt{\frac{d}{m}}$ factor in the second term is only an artefact of the analysis as explained in \cref{sec:rsc}. It is confirmed empirically. 
In this case, leveraging the low rank structure of the parameter matrix through the trace norm regularization would lead to a $\sqrt{\frac{r}{d}}$ improvement over single-task learning and the trace norm regularized estimator would be comparable to the baseline oracle that knows beforehand the $r$-dimensional subspace induced by the parameters.

\subsection{Restricted strong convexity}\label{sec:rsc}

To define the restricted strong convexity condition, we first need to define matrices $U \in \R^{d\times d}$ and $V\in \R^{T\times T}$, such that the SVD of $M^*$ reads $M^* = U \Sigma^* V^\top$, where $\Sigma^* \in \R^{d\times T}$ has only its first $r$ diagonal elements that are non-zero. We now define the following cone of matrices, which is key to the analysis
\begin{equation}\label{eq:cone}
    \cC = \Bigg\lbrace\Delta \in \R^{d\times T} \mid \left\|\Delta_{22}\right\|_{*} \leq 3\bigg\|
    \begin{tikzpicture}[baseline={([yshift=-0.9ex]current bounding box.center)},
every left delimiter/.style={xshift=.75em},
    every right delimiter/.style={xshift=-.75em},
style1/.style={
  matrix of math nodes,
  column sep=0pt,
  row sep =5pt,
  every node/.append style={text width=#1,align=center},
  nodes in empty cells,
  left delimiter=(,
  right delimiter=),
  },
]
\matrix[style1=0.55cm] (1mat)
{
  &\\
  &\\
};
\node  at (1mat-1-1) {$\Delta_{11}$};
\node  at (1mat-2-1) {$\Delta_{21}$};
\node  at (1mat-1-2) {$\Delta_{12}$};
\node  at (1mat-2-2) {$0$};
  \draw[dashed] ($(1mat-2-2.north east)+(-1pt,2pt)$) --++(180:45pt);
  \draw[dashed] ($(1mat-2-2.south west)+(0pt,-4pt)$) --++(90:29pt);
\end{tikzpicture}
\bigg\|_{*}  \text{ where } \Delta = U 
\begin{tikzpicture}[baseline={([yshift=-2.5ex]current bounding box.center)},
every left delimiter/.style={xshift=.75em},
    every right delimiter/.style={xshift=-.75em},
style1/.style={
  matrix of math nodes,
  column sep=0pt,
  row sep =5pt,
  every node/.append style={text width=#1,align=center},
  nodes in empty cells,
  left delimiter=(,
  right delimiter=),
  },
]
\matrix[style1=0.55cm] (1mat)
{
  &\\
  &\\
};
\node  at (1mat-1-1) {$\Delta_{11}$};
\node  at (1mat-2-1) {$\Delta_{21}$};
\node  at (1mat-1-2) {$\Delta_{12}$};
\node  at (1mat-2-2) {$\Delta_{22}$};
\draw[decoration={brace,raise=5pt},decorate,color=vert]
  ($(1mat-1-1.south west)+(0,-1.5pt)$) -- 
  node[left=15pt,font=\tiny,sloped, above=5pt] {$r$} 
  ($(1mat-1-1.north west)+(0,3pt)$);
\draw[decoration={brace,raise=5pt},decorate,color=vert]
  ($(1mat-2-1.south west)+(0,-3pt)$) -- 
  node[font=\tiny, sloped, above=5pt] {$d-r$} 
  ($(1mat-2-1.north west)+(0,1.5pt)$);
\draw[decoration={brace,raise=5pt},decorate,color=vert]
  ($(1mat-1-1.north west)+(5pt,0)$) -- 
  node[above=7pt,font=\tiny] {$r$} 
  ($(1mat-1-1.north east)+(-3pt,0)$);
\draw[decoration={brace,raise=5pt},decorate,color=vert]
  ($(1mat-1-2.north west)+(3pt,0)$) -- 
  node[above=7pt,font=\tiny] {$T-r$} 
  ($(1mat-1-2.north east)-(5pt,0)$);
  \draw[dashed] ($(1mat-2-2.north east)+(-1pt,2pt)$) --++(180:45pt);
  \draw[dashed] ($(1mat-2-2.south west)+(0pt,-4pt)$) --++(90:29pt);
\end{tikzpicture} 
V^\top  \Bigg\rbrace .
\end{equation}
Note that matrices of the form $U     
\begin{tikzpicture}[baseline={([yshift=-0.9ex]current bounding box.center)},
every left delimiter/.style={xshift=.75em},
    every right delimiter/.style={xshift=-.75em},
style1/.style={
  matrix of math nodes,
  column sep=0pt,
  row sep =5pt,
  every node/.append style={text width=#1,align=center},
  nodes in empty cells,
  left delimiter=(,
  right delimiter=),
  },
]
\matrix[style1=0.55cm] (1mat)
{
  &\\
  &\\
};
\node  at (1mat-1-1) {$\cdot$};
\node  at (1mat-2-1) {$\cdot$};
\node  at (1mat-1-2) {$\cdot$};
\node  at (1mat-2-2) {$0$};
  \draw[dashed] ($(1mat-2-2.north east)+(-1pt,2pt)$) --++(180:45pt);
  \draw[dashed] ($(1mat-2-2.south west)+(0pt,-4pt)$) --++(90:29pt);
\end{tikzpicture}
V^\top$ are of rank at most $2r$ with the block dimensions given in \cref{eq:cone}. Any matrix in $\cC$ is thus close to low-rank, since a submatrix of rank $2r$ counts for a significant amount of its nuclear norm.
We now define the linear operator $\mathcal{L}:\R^{d\times T} \to \R^{m\times T}$ 
\begin{gather*}
\mathcal{L}: M \mapsto\frac{1}{\sqrt{mT}} (\langle x_{i}^{t},
M^{(t)} \rangle)_{\substack{1\leq i\leq m\\1\leq t \leq T}}, 
\end{gather*}
which is lower bounded in norm over $\cC$ with high probability by \cref{lemma:rsc} below.
\begin{lemm}[Restricted strong convexity]\label{lemma:rsc}
With probability larger than $1-2e^{-cr(d+T)}-2Te^{-d}$, the operator $\cL$ satisfies
\begin{equation}\label{eq:rsc}
    \| \cL(\Delta) \|_F^2 \geq \frac{c_0}{T} \| \Delta\|_{F}^2 - \frac{c_1 rd(d+T)}{m^2T} \max_{t \in [T]} \| \Delta^{(t)} \|_2^2\ln\pr{\frac{dT}{m}} \qquad \text{for all } \Delta\in\cC,
\end{equation}
where $c,c_0$ and $c_1$ are positive universal constants.
\end{lemm}
This condition generalizes the restricted eigenvalue condition, which is used to prove estimation guarantees of the Lasso or Dantzig estimator in the problem of sparse linear regression \citep{van2009conditions}. It is weaker than the restricted isometry property, which does not hold in the multi-task setting \citep{rohde2011estimation}.

\medskip

The proof of \cref{lemma:rsc} is deferred to \cref{sec:rscproof}. 
%
%
Lemma~11 by \citet{tripuraneni2021provable} states a similar condition on the subset of matrices of rank at most $2r$. Besides correcting minor errors in its proof, we extend this condition to the set $\cC$, which is much larger: its $\eps$-covering number scales exponentially with $\frac{1}{\eps^2}$. As a consequence, our bound includes an additional $\frac{d}{m}$ factor in the last term.
Although we believe this $\frac{d}{m}$ factor to be an artefact of the analysis, covering arguments might not yield better bounds, since the considered subset for the restricted strong convexity condition is much larger here. We let the investigation of more advanced techniques as future work.

\subsection{Effective noise level}\label{sec:noise-level}

The next lemma bounds the effect of label noise on the prediction. 
\begin{lemm}[Effective noise level]\label{lemma:noise-level}
If $T = \Omega(d)$, then for universal positive constants $c_0,c_1,c_2$, with probability at least $1-(2T+c_0)e^{-c_1 d}$:
\begin{gather}\label{eq:noise-level}
\bigg| \frac{1}{mT} \sum_{(i,t)\in[m]\times[T]} \varepsilon_i^t \langle
x_i^t, M^{(t)} \rangle \bigg| \leq \tau \| M\|_{*} \qquad \text{uniformly for all } M \in \mathbb{R}^{d\times T}, 
\end{gather}
where $\tau=\frac{c_2\sigma}{\sqrt{T}}
\sqrt{\frac{T + \nicefrac{d^2}{m}}{mT}}$.
\end{lemm}
The proof is given in \cref{proof:noise-level} and uses concentration results on heavy tailed distributions \citep{bakhshizadeh2020sharp} to bound the operator norm of the matrix $\frac1{mT}\sum_{(i,t)\in[m]\times[T]}\eps_i^t e_t (x_i^t)\tp$.

\medskip

For ``nice'' fixed features $x_i^t$, the best possible bound on $\tau$ is of order $\frac{\sigma}{\sqrt{T}}
\sqrt{\frac{T + d}{mT}}$ using classical bounds on the spectral norm of Gaussian matrices. Our bound is thus tight, up to the $\frac{d^2}{m}$ term, which is actually due to the randomness of the features $x_i^t$. Because of the random design, we cannot directly use Bernstein inequality but instead use concentration on heavy tailed distributions.
In any case, when the number of tasks is large enough, the $T$ term prevails over $\frac{d^2}{m}$ and the bound becomes similar to the easier setting of fixed features.

\subsection{Proof of \cref{thm:estimation1}}\label{sec:mainproof}

The proof assumes that \cref{eq:rsc,eq:noise-level} both hold, which happens with high probability thanks to \cref{lemma:rsc,lemma:noise-level}. 
By definition, the estimator $\hat{M}$ minimizes the objective function in $\mathcal{W}$, to which $M^*$ belongs, thanks to \cref{ass:diversity}. In particular:
\begin{equation}\label{eq:optimal}
 \frac{1}{mT} \sum_{(i,t)\in[m]\times[T]} (y_i^t-\langle x_i^t, \hat{M}^{(t)} \rangle)^2 + \lambda \| \hat{M}\|_{*} \leq  \frac{1}{mT} \sum_{(i,t)\in[m]\times[T]} (y_i^t-\langle x_i^t, M^{*\,(t)} \rangle)^2 + \lambda \| M^*\|_{*}.
\end{equation}
Using simple manipulations, this inequality is equivalent for $\hDelta = \hat{M}-M$ to
\begin{equation*}
\|\cL(\hDelta)\|^2 \leq \frac{2}{mT}\sum_{(i,t)\in[m]\times[T]} \eps_i^t \langle x_i^t, \hDelta^{(t)}\rangle + \lambda\pr{\| M^*\|_{*} - \|\hat{M}\|_{*}}
\end{equation*}
Thanks to \cref{eq:noise-level}, this becomes
$
\|\cL(\hDelta)\|^2 \leq 2\tau\|\hDelta\|_* + \lambda\pr{\| M^*\|_{*} - \|\hat{M}\|_{*}}$.\\
With the SVD decomposition $M^* = U \Sigma^* V^\top$, we can decompose $\hDelta = \hDelta_1 + \hDelta_2$ where
\begin{equation*}
\hDelta = U \begin{pmatrix}
  \hDelta_{11}  & \hDelta_{12}\\
  \hDelta_{21}  & \hDelta_{22}
\end{pmatrix} V^\top; \quad 
\hDelta_1 = U \begin{pmatrix}
  \hDelta_{11} &  \hDelta_{12}\\
  \hDelta_{21} &  0
\end{pmatrix} V^\top \quad\text{and} \quad 
\hDelta_2 = U \begin{pmatrix}
  0 & 0\\
  0 & \hDelta_{22}
\end{pmatrix} V^\top.
\end{equation*}
Using the triangle inequality, $\|\hat{M}\|_{*} \geq \|M^* + \hDelta_2\|_{*} - \|\hDelta_1\|_{*}$. Moreover, as $M^*$ and $\hDelta_2$ have orthogonal column and rowspaces, $\|M^* + \hDelta_2\|_{*} = \|M^*\|_* + \|\hDelta_2\|_{*}$. \cref{eq:optimal} then leads with the choice $\lambda=4\tau$ to
\begin{equation*}
\|\cL(\hDelta)\|^2 \leq 6\tau\|\hDelta_1\|_* - 2\tau\|\hDelta_2\|_{*}.
\end{equation*}
From there, we can first observe that $3\|\hDelta_1\|_*\geq \|\hDelta_2\|_*$, i.e. $\hDelta \in \cC$.
Moreover, since $\hDelta_1$ is of rank at most $2r$, $\|\hDelta_1\|_* \leq \sqrt{2r}\|\hDelta_1\|_F\leq \sqrt{2r}\|\hDelta\|_F$, i.e.
\begin{equation}\label{eq:upperbound}
\|\cL(\hDelta)\|^2 \leq 6\tau\sqrt{2r}\|\hDelta\|_F.
\end{equation}
Now since $\hDelta \in \cC$, \cref{eq:rsc} directly gives
\begin{equation}\label{eq:lowerbound}
\|\cL(\hDelta)\|^2  \geq \frac{c_0}{T} \| \hDelta\|_{F}^2 - \frac{c_1 rd(d+T)}{m^2T} \max_{t \in [T]} \| \hDelta^{(t)} \|_2^2\ln\pr{\frac{dT}{m}}.
\end{equation}
Moreover, as $\hat{M}$ and $M^*$ are both in $\mathcal{W}$, $\max_t\| \hDelta^{(t)} \|_2^2 \leq 4 C$. Combining \cref{eq:upperbound,eq:lowerbound}, this yields
\begin{equation*}
\frac{c_0}{T} \| \hDelta\|_{F}^2 - 6\tau\sqrt{2r}\|\hDelta\|_F - \frac{4Cc_1 rd(d+T)}{m^2T} \ln\pr{\frac{dT}{m}} \leq 0.
\end{equation*}
Note that the left expression is a $2$-degree polynomial in $\| \hDelta\|_{F}$. To be non-positive, it requires \cref{eq:main} to hold, which concludes the proof.
 
\section{Meta-learning: transfer on a new task}\label{sec:meta}

%
Meta-learning is of significant interest in modern applications, where the knowledge acquired on previous tasks is transferred to a single new task.
The objective of the meta-learning setting is to estimate the parameters of a new single task, based on the regression obtained on the previous $T$ tasks.
This section provides estimation error guarantees for this setting.

\medskip

In the following, we use the decomposition $M^* = B \alpha$ where $B \in \R^{d\times r}$ and $\alpha\in \R^{r\times T}$ such that $B^\top B = I_r$\footnote{$B$ corresponds to the $r$ first columns of $U$ defined in \cref{sec:rsc}, up to any rotation of $\R^{r}$.}. Now consider the matrix $\tilde{M}$, defined as the rank $r$ matrix which is the closest to~$\hat{M}$:
\begin{equation}\label{eq:projrank}
    \tilde{M} \in \argmin_{\substack{L \in \R^{d\times T}\\ \mathrm{rank(L)}\leq r}} \|L - \hat{M}\|_F.
\end{equation}
$\tilde{M}$ can be computed from the SVD of $\hat{M}$ by keeping its $r$ largest singular values. Now decompose $\tilde{M}$ as $\tilde{M} = \tilde{B} \tilde{\alpha}$ where $\tilde{B} \in \R^{d\times r}$, $ \tilde{\alpha}\in \R^{r\times T}$ and $\tilde{B}^\top \tilde{B} = I_r$.

\medskip

The meta-learning setting considers a $T+1$-st task, with $m$ observations $(x_i^{T+1}, y_i^{T+1})$ generated as described in \cref{sec:model}. Using the estimated subspace matrix $\tilde{B}$, we compute the least squares estimate
\begin{equation*}
    \tilde{\alpha}_{T+1} \in \argmin_{\theta \in \R^r} \sum_{i=1}^{m} \pr{y_i^{T+1}-\langle \tilde{B}\theta, x_i^{T+1} \rangle}^2.
\end{equation*}
The idea behind meta-learning is that even for a small number of observations $m$ on this single task, the parameters vector $M^{*\,(T+1)}$ can still be well estimated using the $T$ previous tasks.

\medskip

Before bounding the error $\|\tilde{B}\tilde{\alpha}_{T+1} - M^{*\,(T+1)}\|_2$, \cref{coro:angle} bounds the angles between the $r$-dimensional subspaces corresponding to $B$ and $\tilde{B}$.
It is adapted from \citet[][Lemma~16]{tripuraneni2021provable} to take into account the additional error that might appear from the low-rank projection step given in \cref{eq:projrank}. Its proof is given in \cref{sec:angleproof}.
\begin{lemm}\label{coro:angle}
For $B$ and $\tilde{B}$ defined as above:
\begin{equation*}
    \sin^2\theta\pr{B,\tilde{B}} \leq \frac{4r\|\hat{M}-M^*\|_F^2}{T\nu},
\end{equation*}
where $\nu = r \lambda_r\pr{\frac{M^* M^{*\,\top}}{T}}$ and $\theta\pr{B,\tilde{B}}$ is the principal angle between the subspaces corresponding to the orthogonal projections $B$ and $\tilde{B}$.
\end{lemm}
The variable $\nu$ in \cref{coro:angle} is introduced for clarity. Note that if the tasks are \textit{rich}\footnote{For example if $\|M^*\|_F^2 = \Omega(T)$ and $\frac{\lambda_1(M^* M^{*\,T})}{\lambda_r(M^* M^{*\,T})}=\bigO{1}$}, then $\nu = \Omega(1)$. \cref{coro:angle} states that the $r$-dimensional feature subspace is well estimated if the parameter matrix is well estimated. From there, Theorem~4 by \citet{tripuraneni2021provable} allows to bound the error of the estimation of the parameter vector for the $T+1$-th task. \cref{coro:metabound} below is stated in the balanced case, where the new task also has $m$ samples. It can easily be extended to the unbalanced case.
\begin{coro}\label{coro:metabound}
If \cref{eq:main} holds and $m = \Omega(r\log(r))$, then for $\tilde{B}$ and $\tilde{\alpha}_{T+1}$ defined above, with probabilty at least $1-\frac{c}{m^{100}}$:
\begin{equation*}
\|\tilde{B}\tilde{\alpha}_{T+1} - M^{*\,(T+1)}\|_2^2 \leq c'C\sigma^2 \frac{r^2}{m\nu}\pr{\frac{d^2}{mT}+\log(m)} + c'C^2\frac{r^2d}{m^2\nu}\pr{\frac{d}{T}+1},
\end{equation*}
where $c$ and $c'$ are universal constants.
\end{coro}
The squared estimation error on a new task is roughly bounded by $\frac{r^2d}{m^2}$. In contrast, it is known that a (single task) linear regression on this task would lead to an estimation error of order $\frac{d}{m}$ \citep{hsu2012random}, \ie leveraging the low rank structure of the parameters and the past observations leads at least to an improvement $\frac{r^2}{m}$ for the trace norm regularized estimator. 

\medskip

As a comparison, the best known bounds for Burer-Monteiro factorization \citep{thekumparampil2021sample} and Method of Moments \citep{tripuraneni2021provable} on a new task are of order $\sigma^2\frac{r}{m}$ for a large number of tasks, thus being comparable to the oracle baseline. These approaches yield better meta-learning bounds, but require respectively $m=\Omega(\log(T))$ or a spherically symmetric feature distribution.

The discrepancy between these bounds is due to different analyses. Our meta-learning bound directly derives from the multi-task bound using \cref{coro:angle}, while the tight analyses of Burer-Monteiro (with alternate minimization) and Method of Moments directly bound the principal angle between the estimated subspaces.
\cref{coro:angle} indeed considers the unrealistic worst case, where the entirety of the estimation error of the matrix $M^*$ is due to the error in subspace estimation. 
%
As a consequence, \cref{coro:angle} leads to a loose bound on the subspace angle of $\bigO{\frac{r^2d}{m^2}}$, while this angle clearly goes to $0$ when the number of tasks grows to infinity in the simulations of \cref{sec:simulations}.
Showing it actually converges to $0$ for a large number of tasks is left for future work and should be done by directly bounding the subspace angle. Such a result would then lead to an optimal meta-learning bound for trace norm based methods, without any further requirement on the number of observations per task or on the feature distribution. 

\section{Experiments}\label{sec:simulations}

This section compares empirically different multi-task methods on
synthetic datasets and discusses the practical aspects of their
implementation. 

\subsection{Simulations}

With the exception of \cref{fig:ntasks-adversarial}, our experimental setup follows that of \citet{tripuraneni2021provable,thekumparampil2021statistically}. More specifically, we set $d=100,r=5$ and sample $x_i^t\simiid\cN(0,I_d)$, $\eps_i^t\simiid\cN(0,\sigma^2)$.
The following experiments study how the normalized Frobenius distance $\norm{\hat{M}-M^*}_F/\sqrt{T}$ and the angle between the subspaces $\sin \theta\pr{B,\tilde{B}}$ behave when varying the number of tasks $T$, the size of each task $m$ and the label noise level $\sigma$. In all experiments the markers show the average and the shaded area shows the standard deviation over 12 independent runs. The code is available in \url{github.com/MichaelKonobeev/meta} and additional experimental details are given in \cref{app:simulations}.

\medskip

In this section, \texttt{altmin} corresponds to the alternating minimization algorithm \citep{thekumparampil2021sample}; 
\texttt{bm} and \texttt{mom} respectively correspond to Burer-Monteiro factorization and the Method of Moments \citep{tripuraneni2021provable}; 
\texttt{nuc} corresponds to the nuclear norm regularized estimator.  
Additionally, we implement two other baselines labeled \texttt{single}, which implements independent least squares regression of the $d$-dimensional parameter vectors (the columns of $M$) for each task;
and \texttt{oracle} which knows the ground-truth subspace $B$ and only estimates the matrix $\alpha$ by performing independent least squares regression for each task in the projected $r$-dimensional space. %
Some algorithms perform very poorly for certain values of $T,m,\sigma$ in Frobenius distance.
In such cases, we omit displaying these points in the figures for the sake of readability.

\begin{figure}[ht]
\centering
\includegraphics[width=0.85\linewidth,trim=2cm 0cm 0cm 0cm, clip]{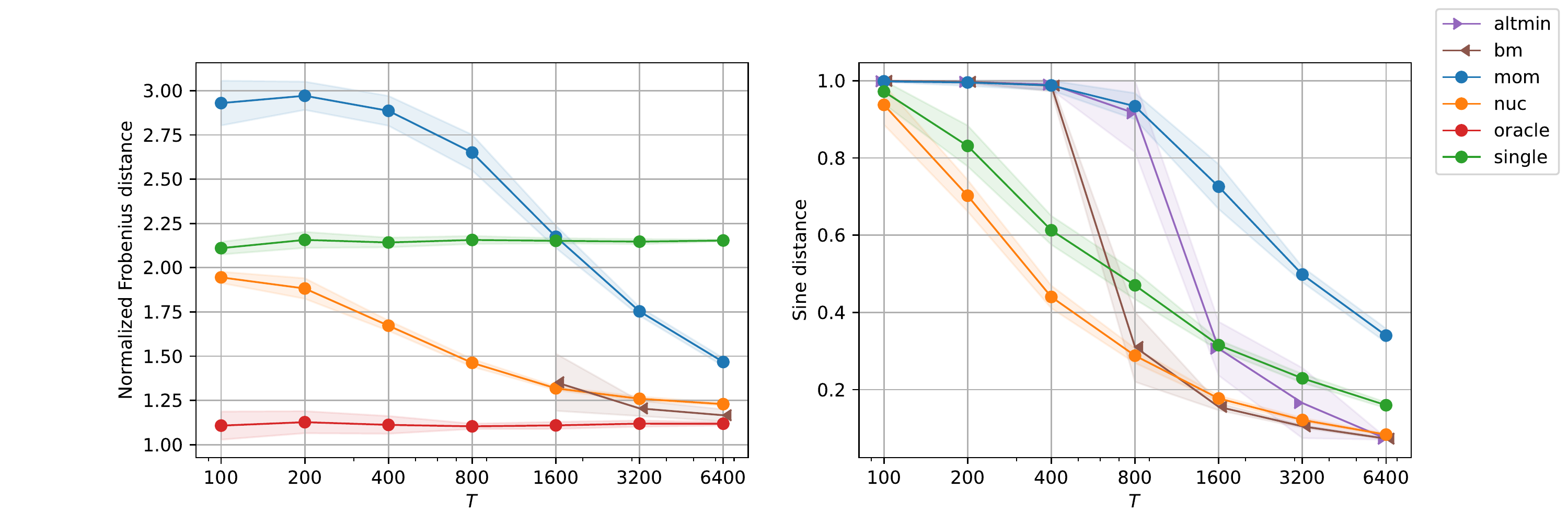}
\caption{Evolution of the estimation error with the number of tasks $T$ for $m=10$.}
\label{fig:ntasks-medium}
\end{figure}
\begin{figure}[ht]
\centering
\includegraphics[width=0.85\linewidth,trim=2cm 0cm 0cm 0cm, clip]{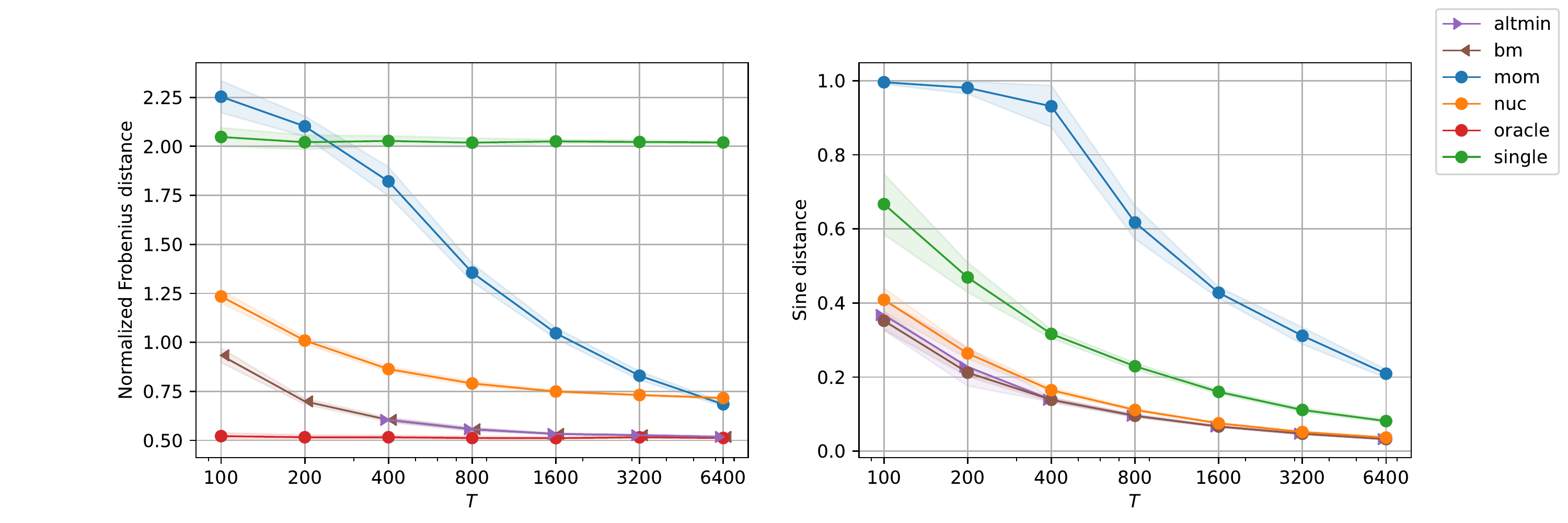}
\caption{Evolution of the estimation error with the number of tasks $T$ for $m=25$.}
\label{fig:ntasks-large}
\end{figure}
\cref{fig:ntasks-medium} displays the normalized Frobenius distance and the angle between the subspaces when varying the number of tasks $T$ for a small number of observations per task ($m=10$). 
We observe that \texttt{nuc} outperforms the other algorithms in both distances. When the number of tasks becomes large, \texttt{altmin} and \texttt{bm} nevertheless become comparable in sine distance (even in Frobenius distance for \texttt{bm}). 
%
%

On the other hand, \texttt{altmin} and \texttt{bm} outperform \texttt{nuc} with larger $m$ ($m=25$) as shown in \cref{fig:ntasks-large}. The regime of interest for \texttt{nuc} thus seems to be for small values of $m$, as can be seen in \cref{fig:task-size} below. These empirical results confirm the different known theoretical bounds.
\begin{figure}[ht]
\centering
\includegraphics[width=0.85\linewidth,trim=2cm 0cm 0cm 0cm, clip]{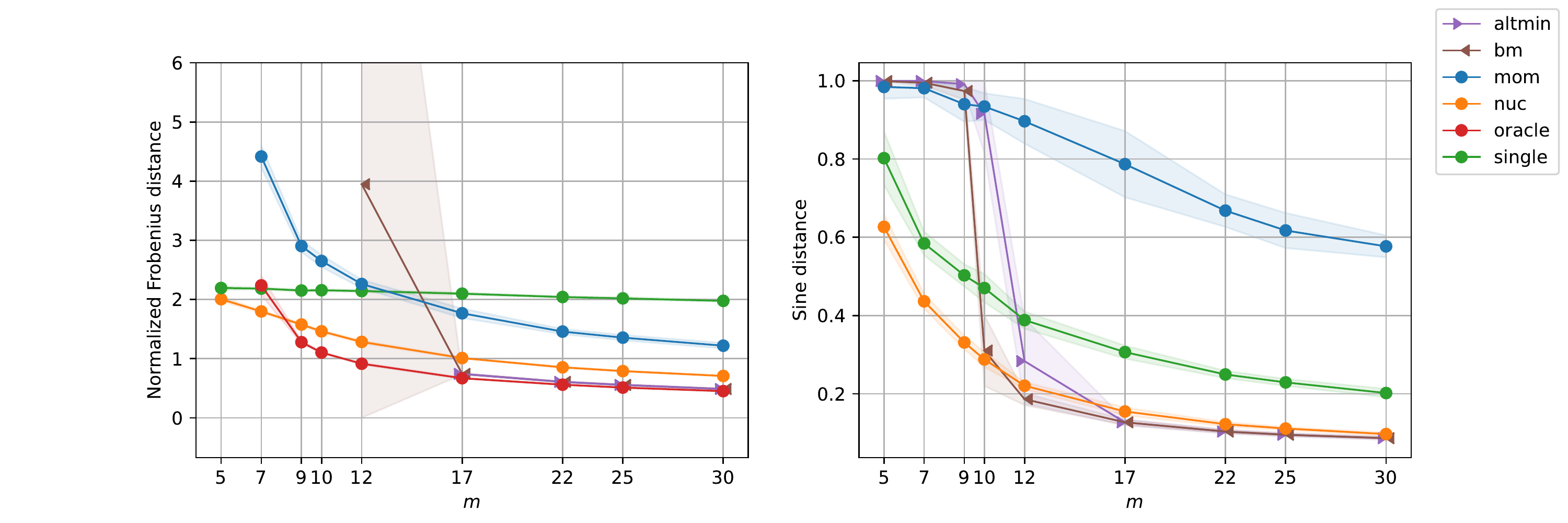}
\caption{Evolution of the estimation error with the task size $m$ for $T=800$.}
\label{fig:task-size}
\end{figure}

\medskip

Method of Moments is largely outperformed by the other algorithms, unless the number of tasks is very large.
Furthermore, as highlighted in \cref{fig:ntasks-adversarial} below, it might fail for non-spherically symmetric distributions\footnote{Further details on the chosen feature and parameters distribution are given in \cref{app:simulations}.}.
In particular, the largest principal angle between its estimated subspace and the real one remains equal to $\frac{\pi}{2}$. Method of Moments yet outperforms \texttt{single} in Frobenius distance, because it still manages to learn some of the directions of the $r$-dimensional subspace.

\begin{figure}[ht]
\centering
\includegraphics[width=0.85\linewidth,trim=2cm 0cm 0cm 0cm, clip]{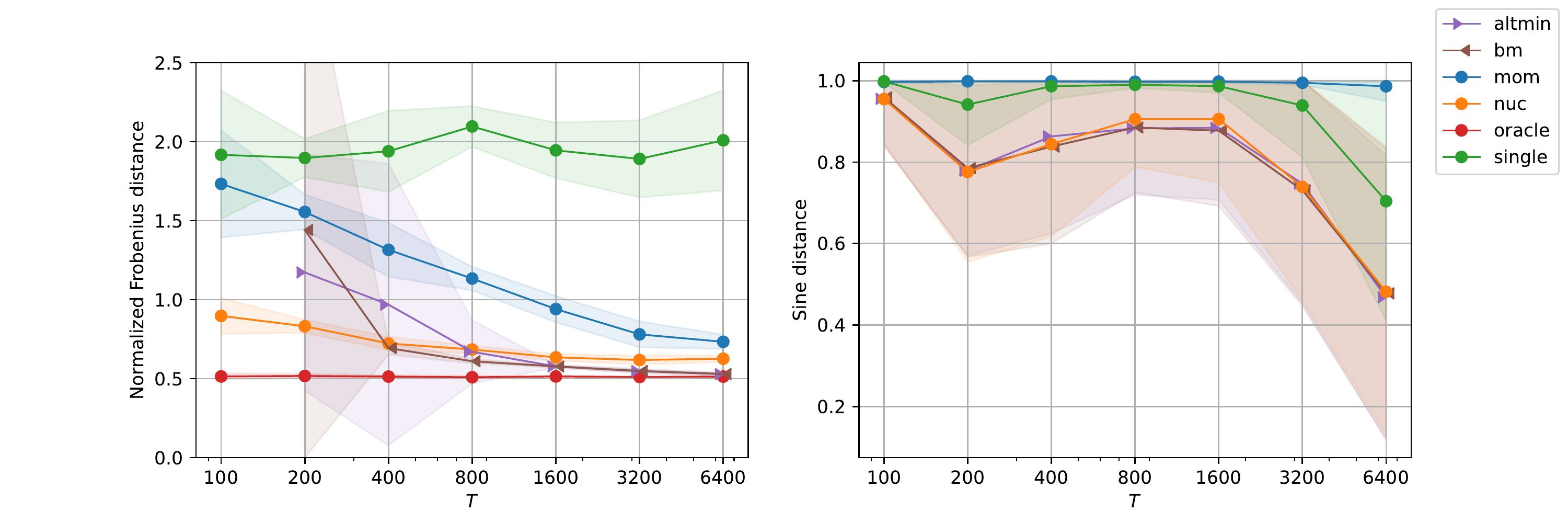}
\caption{Evolution of the estimation error with the number of tasks $T$ for $m=25$ and a non-spherically symmetric feature distribution.}
\label{fig:ntasks-adversarial}
\end{figure}

Finally, \cref{fig:label-scale} studies the evolution of the estimation error with the level noise $\sigma$. Algorithms \texttt{altmin} and \texttt{bm} do not perform well in Frobenius norm here and are not displayed for a clearer figure.
Although \texttt{nuc} scales better with $\sigma$ than \texttt{mom}, both methods seem to not recover the exact parameters matrix in the noiseless setting, confirming the $(1+\sigma)$ dependence in \cref{thm:informal} for the former.

\begin{figure}[ht]
\centering
\includegraphics[width=0.85\linewidth,trim=2cm 0cm 0cm 0cm, clip]{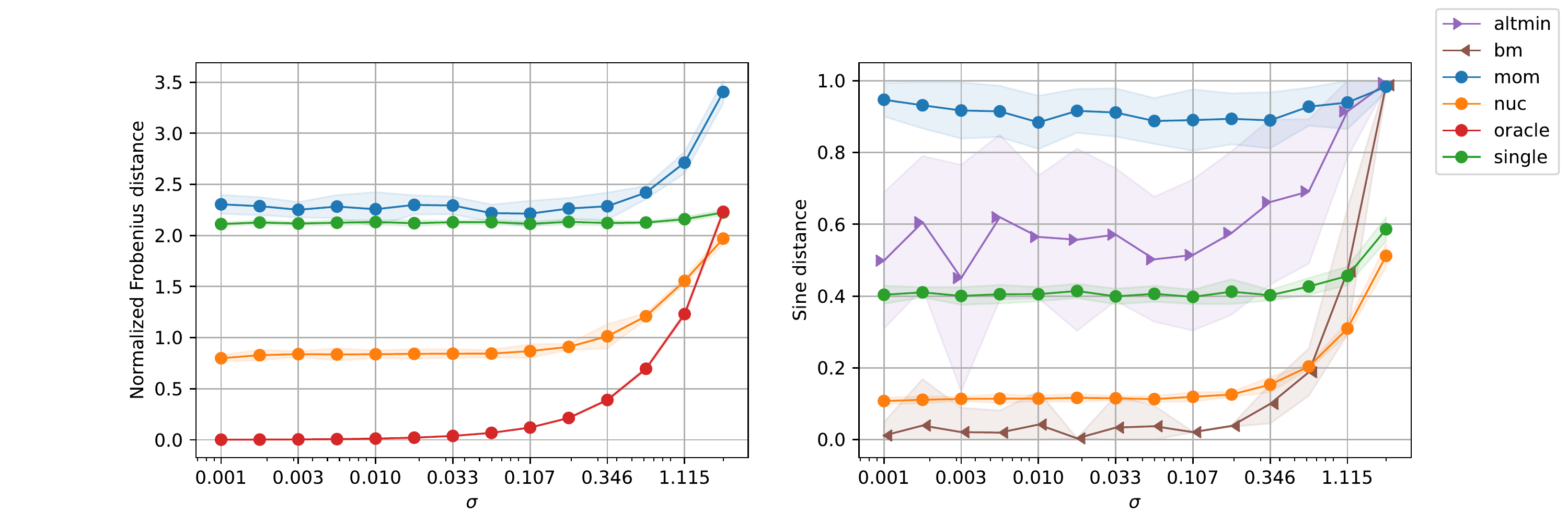}
\caption{Evolution of the estimation error with the noise level $\sigma$ for $T=800$ and $m=10$.}
\label{fig:label-scale}
\end{figure}
\medskip

As explained in \cref{sec:meta}, \texttt{nuc} is comparable to \texttt{altmin} in subspace estimation in the experiments. This suggests that the meta-learning bound of \cref{coro:metabound} is not tight and it needs further investigation.

\subsection{Numerical complexity} \label{sec:complexity}

A local minimum of the Burer-Monteiro factorization can be approximated up to $\eps$ using first order methods such as gradient descent \citep{jin2017escape} in $\bigO{\frac{1}{\eps^2}}$ steps. Computing the gradient of the objective function here requires at each step a multiplication between $d\times r$ and $r\times T$ matrices, leading to a total complexity of $\bigO{\frac{rdT}{\eps^2}}$.

\medskip

For trace norm regularization, algorithms approximating a solution of \cref{eq:nuc-opt-prob} up to $\eps$ in $\bigO{\frac{1}{\sqrt{\eps}}}$ steps exist \citep{ji2009accelerated,toh2010accelerated}, but  require an SVD computation (of complexity $dT^2$) at each step, thus leading to a considerable numerical complexity for a large number of tasks. Instead, \citet{jaggi2010simple} propose an algorithm approximating the solution of the equivalent constrained problem with a numerical complexity $\bigO{\frac{dT}{\eps^2}}$.
%
In practice, it might thus be preferable to consider the constrained problem, as it is cheaper in  computation and the considered optimization programs are equivalent for properly tuned regularization/constraint parameters.

\medskip

On the other hand, Method of Moments is much better in terms of complexity as it only computes a single truncated SVD and then proceeds to $T$ linear regressions with $r$ parameters. In total, its complexity is thus of order $\bigO{rT (d+m^2)}$, largely outperforming the other methods in terms of numerical complexity.
The Method of Moments is thus computationally cheaper than both Burer-Monteiro factorization and trace norm constrained minimization, which end up being similar in computational cost. 

\paragraph{Parameter tuning.}
Another important practical aspect of these algorithms is parameter tuning.
Although trace norm methods require tuning the regularization parameter, \cref{thm:estimation1} needs this parameter to depend on known variables except from the noise level $\sigma$. On the other hand, Burer-Monteiro factorization and Method of Moments both require the knowledge of the rank of the parameters matrix, which is unknown. The parameter tuning is thus easier in practice for trace norm based methods, being one of their main practical advantages in multi-task learning.

\section{Conclusion}

This work proposes the first multi-task estimation error bound when the number of samples per task is very limited. It illustrates the interest of nuclear norm regularization for low rank matrix estimation in this intricate regime and confirms empirically its good performance on synthetic datasets, with respect to the other known methods.
It confirms that learning shared representation is possible with a limited number of samples per task, in the sample linear representation model, grasping insights of the empirical success on learning non-linear representation.

In light of the experiments, the proposed bounds might be improvable, especially in the subspace estimation. Such an improvement is left open for future work and would require more refined analytical tools.   

\acks{We warmly thank Nilesh Tripuraneni and Alexandre Tsybakov for their time and precious feedback.}

\pagebreak
\bibliography{ref.bib}

\pagebreak
\appendix
\section{Multi-task analysis of Method of Moments}\label{sec:MoM}

This section provides a multi-task analysis of the Method of Moments (MoM) algorithm introduced in \citep{tripuraneni2021provable}. \citet{tripuraneni2021provable} only provided a meta-learning type bound (similar to \cref{coro:metabound}) for this algorithm, in the particular case of a Gaussian feature distribution. 
This section adapts this result to derive a multi-task bound for \cref{algo:MoM}, with any spherically symmetric feature distribution\footnote{A distribution is spherically symmetric if it is invariant under any orthogonal transformation.}.

\begin{algorithm2e}[htp]
\caption{Method of Moments}
\label{algo:MoM}
\DontPrintSemicolon 
\KwIn{$(x_i^t, y_i^t)_{i,t} \in \pr{\R^d \times\R}^{mT} $}
\vspace{1em}
$\hat{B_1}\hat{D}_1\hat{B}_1\tp \gets $ top $r$-SVD of $\sum_{i=1}^m\sum_{t=\lceil \frac{T}{2} \rceil+1}^{T} y_i^t x_i^t (x_i^t)\tp$ \tcp*{$\hat{B}_1 \in \R^{d\times r}$}
$\hat{\alpha}_1 \gets \argmin_{\alpha \in \R^{r\times \lceil \frac{T}{2} \rceil}} \sum_{i=1}^m\sum_{t=1}^{\lceil \frac{T}{2} \rceil} \pr{y_i^t-\langle \hat{B}_1\alpha^{(t)}, x_i^t\rangle}^2$\;
\vspace{1em}
$\hat{B_2}\hat{D}_2\hat{B}_2\tp \gets $ top $r$-SVD of $\sum_{i=1}^m\sum_{t=1}^{\lceil \frac{T}{2} \rceil} y_i^t x_i^t (x_i^t)\tp$ \;
$\hat{\alpha}_2 \gets \argmin_{\alpha \in \R^{r\times T-\lceil \frac{T}{2} \rceil}} \sum_{i=1}^m\sum_{t=\lceil \frac{T}{2} \rceil+1}^{T} \pr{y_i^t-\langle \hat{B}_2\alpha^{(t)}, x_i^t\rangle}^2$\;
\vspace{1em}
\Return{$\large( \hat{B}_1 \hat{\alpha}_1, \hat{B}_2 \hat{\alpha_2} \large)$}\;
\end{algorithm2e}

This algorithm directly estimates the low dimensional subspace by aggregating all the observations (for different tasks) together. This estimation relies on the fourth moments of the features, hence its name.
Note that the considered MoM algorithm is slightly modified with respect to the original one: we split the tasks in two batches and estimate a batch parameters using the estimated subspace on the other batch. This trick is used to get independence between the estimated subspace and the estimated parameters, which allows to apply Theorem~4 by \citet{tripuraneni2021provable}. This result could not be directly used without splitting the data.

\medskip

\cref{thm:MoM} bounds the estimation error of MoM. Its proof is deferred to \cref{sec:MoMproof}.
\begin{thm}\label{thm:MoM}
Consider the setting defined in \cref{sec:model}. Denote $M_1^* \in \R^{d\times \lceil \frac{T}{2} \rceil}$ the matrix corresponding to the first $\lceil \frac{T}{2} \rceil$ columns of $M^*$ and $M_2^*$ the matrix corresponding to the remaining columns. 
Additionally, if the feature distribution is spherically symmetric and  $m \geq c r \log(r)$ and $mT \geq c \frac{\log^6(dmT)\pr{\sigma^4+C^2}r^2}{\kappa^2 \nu^2}d$, then with probability at least $1-\frac{2}{m^2T^2}$, the estimator returned by \cref{algo:MoM} verifies:
\begin{equation}
    \|\hat{M} - M^* \|_F^2 \leq c' \sigma^2\frac{r}{m}T\log(mT) + c'C\log^6(dmT)r^2\frac{\sigma^4+C^2}{\kappa^2 \bar{\nu}^2}\frac{d}{m},
\end{equation}
where $\bar{\nu}=r\min\pr{\lambda_r\pr{\frac{M_1^* M_1^{*\, \top}}{T}}, \lambda_r\pr{\frac{M_2^* M_2^{*\, \top}}{T}}}$, $\kappa = \E[\langle e_1, x\rangle^4] - {\E[\langle e_1, x\rangle^2\langle e_2, x\rangle^2]>0}$, $c$ and $c'$ are universal positive constants.
\end{thm}
Similarly to the term $\nu$ in \cref{sec:meta}, $\bar{\nu} = \Omega(1)$ for rich tasks. For large $T$, the first term then prevails and \cref{thm:MoM} recovers (up to a logarithmic term) the estimation error of a linear regression on an $r$-dimensional subspace. The method of moments algorithm however requires the feature distribution to be spherically symmetric and might fail otherwise as illustrated by \cref{fig:ntasks-adversarial} in \cref{sec:simulations}.

Moreover, the method of moments is significantly worse than Burer-Monteiro factorization and trace norm regularization in practice, hence the need for a better understanding of these methods.

\section{Experimental details}\label{app:simulations}

In \cref{fig:ntasks-medium,fig:ntasks-large,fig:task-size,fig:ntasks-adversarial}, we choose $\sigma=1$.
Except for \cref{fig:ntasks-adversarial}, the ground-truth matrix $M^*$ is generated as follows. We first take $B$ as the top $r$ left singular vectors of a $d\times d$ random matrix whose entries are i.i.d. and drawn from $\cN(0,1)$. Next, we sample $\alpha\in\R^{r\times T}$ with $\alpha_{i,j}\simiid\cN(0,1)$ and compute the resulting matrix $M^*=B\alpha$. 

\medskip

In all experiments except the one in \cref{fig:label-scale} the value of the regularization coefficient for \texttt{nuc} is chosen as $\lambda=\frac{\sigma}{\sqrt{T}}\sqrt{\frac{T+\nicefrac{d^2}{m}}{mT}}$. 
In the experiment with varying $\sigma$ in \cref{fig:label-scale}, we found that for small values of $\sigma$ the regularization coefficient computed using the above formula becomes too small. We thus use validation sets of size $0.2m$ for each task to find the value of the regularization coefficient via grid search.

\medskip

For the setting in \cref{fig:ntasks-adversarial} of \cref{sec:simulations}, the feature distribution samples independently each coordinate. The $k$-th coordinate is given by $\langle x_i^t, e_k\rangle=\cos(\frac{k}{d}\frac{\pi}{2})\xi_{i,k}^t + \sin(\frac{k}{d}\frac{\pi}{2})\eta_{i,k}^t$, where $\xi_{i,k}^t\simiid\cU[-\sqrt{3},\sqrt{3}]$ and $\eta_{i,k}^t\simiid\cN(0,1)$.
The $r$-dimensional parameters are generated by first sampling $r\times r$ matrix $\alpha'$ with each element $(\alpha')_{i,j}\simiid\cN(0,1)$ and next completing this to a $T\times r$ matrix by sampling $T-r$ vectors uniformly at random among the set of columns of the matrix $\alpha'$.

This choice of feature distribution is to ensure different fourth moments of the features along different directions, while this choice of parameters avoids that the average parameters matrix is too well behaved.

\section{Proofs}

This section provides all the proofs deferred from the main text. 

\subsection{Proof of \cref{lemma:rsc}}\label{sec:rscproof}

In the whole section, we assume $m = \bigO{d}$. Adapting \cref{lemma:operatornorm} below, we can actually show if $m\gtrsim d$ that with probability at least $1-2T\exp\pr{-d}$:
\begin{equation*}
\frac{c}{\sqrt{T}}\leq \min_{\Delta \in \R^{d\times T}}\frac{\| \cL(\Delta) \|_F}{\|\Delta\|_F} \leq \max_{\Delta \in \R^{d\times T}}\frac{\| \cL(\Delta) \|_F}{\|\Delta\|_F} \leq  \frac{c'}{\sqrt{T}},
\end{equation*}
thus leading to a restricted isometry condition that also implies \cref{lemma:rsc}.

\medskip

The proof first bounds the operator norm of $\mathcal{L}$.
\begin{lemm}\label{lemma:operatornorm}
With probability at least $1-2T\exp\pr{-d}$:
\begin{equation*}
\max_{\Delta \in \R^{d\times T}}\frac{\| \cL(\Delta) \|_F}{\|\Delta\|_F} \leq  c\frac{\sqrt{d} + \sqrt{m}}{\sqrt{mT}},
\end{equation*}
for some universal constant $c$.
\end{lemm}

\begin{proof}
Note that by definition of $\cL$, for any $\Delta\in\R^{d\times T}$:
\begin{equation}
\label{eq:taskdecomposition}
\|\cL(\Delta) \|_2^F = \frac{1}{mT}\sum_{t=1}^T \| X_t \Delta^{(t)} \|_2^2,
\end{equation}
where $\Delta^{(t)}$ is the $t$-th column of $\Delta$ and $X_t = [x_1^t, \ldots, x_m^t]^\top \in \R^{m \times d}$.

By design of the setting, $X_t$ is a matrix with independent, mean zero, sub-gaussian isotropic random vectors in $\R^d$. Theorem~4.6.1 by \citet{vershynin2018high} directly yields for a universal constant $c$ that
\begin{equation*}
\| X_t \|_2 \leq  (1+c)\sqrt{d} + c\sqrt{m} \qquad \text{ with probability at least } 1-2e^{-d}.
\end{equation*}
Taking a union bound over all tasks, with probability at least $1-2Te^{-d}$, $\| X_t \|_2 \leq  (1+c)\sqrt{d} + c\sqrt{m}$ for all tasks $t \in [T]$. This finally leads to \cref{lemma:operatornorm} using \cref{eq:taskdecomposition}.
\end{proof}

We then show an RSC condition on the set of low rank matrices similar to Lemma~11 by \citet{tripuraneni2021provable}.

\begin{lemm}\label{lemma:rscrank}
For any $r' \in \mathbb{N}$, with probability at least $1-\exp\pr{c_1 r'(d+T)}$:
\begin{equation*}
\|\mathcal{L}(\Delta)\|_F^2 \geq 3\frac{\|\Delta\|_F^2}{8T} - c' \frac{r'(d+T)}{mT} \max_{t \in [T]} \|\Delta^{(t)}\|^2 \ln\pr{\frac{dT}{m}} \qquad \text{ uniformly over all matrices of rank at most } r',
\end{equation*}
where $c,c', c_1, c_2$ are positive universal constants.
\end{lemm}

\begin{proof}
By homogeneity, it suffices to show this for any matrix in $\Gamma_{r'}$ where
\begin{equation*}
\Gamma_{r'} = \lbrace M \in \R^{d\times T} \mid \mathrm{rank}(M)\leq r' \text{ and } \|M\|_F = 1 \rbrace.
\end{equation*}
For any $\eps\in (0,1)$, we know there exists an $\eps$-covering of $\Gamma_{r'}$ of cardinality smaller than $\pr{\frac{9}{\eps}}^{r'(d+T+1)}$ \citep[][Lemma 3.1]{candes2011tight}. 
For $\eps = \min\left\lbrace \frac{c}{4}\frac{\sqrt{m}}{\sqrt{d}+\sqrt{m}} , \frac{1}{\sqrt{T}}\right\rbrace$ where $c$ is the constant in \cref{lemma:operatornorm}, let $\cN$ be an $\eps$-covering of $\Gamma_{r'}$ of minimal size. By union bound, taking $\log\frac{1}{\delta}$ of order $r'(d+T)\ln\pr{\frac{dT}{m}}$ in \cref{lemma:pointwisersc} of \cref{sec:lemmastau} then states that with probability at least $1-\exp\pr{-c_1 r'(T+d)}$, for all $M\in\cN$:
\begin{equation}\label{eq:concgrid}
\|\cL(M)\|^2 \geq \frac{1}{T} -\frac{c_1}{4}\frac{\max_{t} \|M^{(t)}\|_2}{\sqrt{m}T}\sqrt{r'(d+T)\ln\pr{\frac{dT}{m}}} - c_1^2 \frac{\max_{t}\|M^{(t)}\|^2}{mT}r'(d+T)\ln\pr{\frac{dT}{m}}.
\end{equation}
Also, \cref{lemma:operatornorm} states that with probability at least $1-2T\exp(-d)$,
\begin{equation}\label{eq:operatornorm0}
\|\cL\|_2^2 \leq c\frac{d + m}{mT}.
\end{equation}
Assume in the following that \cref{eq:concgrid,eq:operatornorm0} both hold.
Consider $\Delta \in \Gamma_{r'}$ and decompose as $\Delta = M + A$ where $M \in \cN$ and $\|A\|_F \leq \eps$.

Note that \cref{eq:concgrid} leads to the following inequality, using $a^2 - \frac{ab}{4} - b^2 \geq \frac{7}{8}a^2 - \frac{9}{8}b^2$:
\begin{equation*}
\|\cL(M)\|^2 \geq \frac{7}{8T} - \frac{9c_1^2}{8} \frac{\max_{t}\|M^{(t)}\|^2}{mT}r'(d+T)\ln\pr{\frac{dT}{m}}.
\end{equation*}
Moreover, $\max_{t}\|M^{(t)}\|^2 \leq 2\max_{t}\|\Delta^{(t)}\|^2 + 2\max_{t}\|A^{(t)}\|^2$. As $\eps \leq \frac{1}{\sqrt{T}} \leq \max_{t}\|\Delta^{(t)}\|$, this last inequality yields that $\max_{t}\|M^{(t)}\|^2 \leq 4\max_{t}\|\Delta^{(t)}\|^2$, and we thus have:
\begin{equation}\label{eq:concgrid1}
\|\cL(M)\|^2 \geq \frac{7}{8T} - \frac{9c_1^2}{2} \frac{\max_{t}\|\Delta^{(t)}\|^2}{mT}r'(d+T)\ln\pr{\frac{dT}{m}}.
\end{equation}
Thanks to \cref{eq:operatornorm0} and the choice of $\eps$, we also have
\begin{equation}\label{eq:A}
\|\cL(A)\|^2 \leq \frac{1}{16T}.
\end{equation}
Finally, this leads to
\begin{align*}
\|\cL(\Delta)\|^2 & \geq \pr{\|\cL(M)\|-\|\cL(A)\|}^2 \hspace{0cm} &\text{by triangle inequality}\\
& \geq \frac{\|\cL(M)\|^2}{2} - \|\cL(A)\|^2\hspace{0cm} &\text{using } (a-b)^2 \geq \frac{a^2}{2} - b^2 \\
& \geq \frac{3}{8T} - \frac{9c_1^2}{8} \frac{\max_{t}\|\Delta^{(t)}\|^2}{mT}r'(d+T)\ln\pr{\frac{dT}{m}} \hspace{0cm} & \text{from \cref{eq:concgrid1,eq:A}}.
\end{align*}
\end{proof}

We can now prove \cref{lemma:rsc}. By homogeinity, it suffices to show it for any $\Delta \in \cC$ of Frobenius norm~$1$. 
\cref{lemma:rscrank} yields that with high probability:
\begin{equation}\label{eq:rank}
\|\mathcal{L}(\Delta)\|_F^2 \geq \frac{3\|\Delta\|_F^2}{8T} -  c \frac{r(d+T)}{mT\eps^2} \max_{t \in [T]} \|\Delta^{(t)}\|^2\ln\pr{\frac{dT}{m}} \qquad \text{for all matrices of rank at most } \frac{32}{\eps^2}r.
\end{equation}
Recall that \cref{lemma:operatornorm} states that with high probability,
\begin{equation}\label{eq:operatornorm1}
\|\cL\|_2^2 \leq c\frac{d + m}{mT}.
\end{equation} 
Assume in the following of the proof that both \cref{eq:rank,eq:operatornorm1} hold for $\eps = \frac{1}{4\sqrt{c}}\frac{\sqrt{m}}{\sqrt{d}+\sqrt{m}}$.

Let $\Delta = \tilde{U} \tilde{\Sigma} \tilde{V}^\top$ be the SVD decomposition of $\Delta$, i.e. $\tilde{U}^\top \tilde{U} = I_{d}, \tilde{V}^\top \tilde{V} = I_{T}$ and $\tilde{\Sigma} = \mathrm{diag}(\sigma_1, \ldots, \sigma_{\min(d,T)})$, where the sequence $(\sigma_i)$ is non-increasing.

Denote in the following $\tSigma_{r'} = \mathrm{diag}(\sigma_1, \ldots, \sigma_{r'}, 0, \ldots, 0)$ and $\Delta_{r'} = \tU \tSigma_{r'} \tV^\top.$
Note that by definition of $\cC$, $\|\Delta \|_* \leq 4\sqrt{2r} \|\Delta \|_F$, i.e.
\begin{align*}
\sum_{k=1}^{\min(d,T)} \sigma_k \leq 4 \sqrt{2r}.
\end{align*}
By monotonicity of the sequence, this implies the following decrease rate for the singular values $\sigma_k$:
\begin{equation}\label{eq:sigmarate}
\sigma_k \leq \frac{4\sqrt{2r}}{k} \qquad \text{for any } k \leq \min(d,T).
\end{equation}
From this, it follows:
\begin{align*}
\| \Delta - \Delta_{r'} \|_F^2 & = \sum_{k=r'+1}^{\min(d,T)} \sigma_k^2 \\
& \leq 32r \sum_{k=r'+1}^{\min(d,T)} \frac{1}{k^2} \qquad \text{using \cref{eq:sigmarate}} \\
& \leq \frac{32 r}{r'}\qquad \text{by integral comparison}.
\end{align*}
Fix in the following $r' = \frac{32r}{\eps^2}$ and decompose $\Delta = \Delta_{r'} + \pr{\Delta - \Delta_{r'}}$. Note that $\|\Delta - \Delta_{r'}\|_F \leq \eps$ and $\Delta_{r'}$ is of rank at most $\frac{32}{\eps^2}r$. 
Moreover, the columnspaces of $\Delta_{r'}$ and $\Delta - \Delta_{r'}$ are orthogonal, thanks to the used decomposition. Thanks to that, it follows that $\max_{t \in [T]} \| \Delta_{r'}^{(t)}\|_2 \leq \max_{t \in [T]} \| \Delta^{(t)}\|_2 $. 
From \cref{eq:rank,eq:operatornorm1}, we then have for $\eps = \frac{1}{4\sqrt{c}}\frac{\sqrt{m}}{\sqrt{d}+\sqrt{m}}$.
\begin{gather*}
\|\mathcal{L}(\Delta_{r'})\|_F^2 \geq \frac{3\|\Delta_{r'}\|_F^2}{8T} -  2 \frac{rd(d+T)}{m^2T} \max_{t \in [T]} \|\Delta^{(t)}\|^2\ln\pr{\frac{dT}{m}} \\
\|\cL(\Delta - \Delta_{r'})\|_F^2 \leq \frac{1}{16T}.
\end{gather*}
As $\|\Delta_{r'}\|_F^2 \geq 1-\frac{1}{16}$, we can show similarly to the proof of \cref{lemma:rscrank} that these two inequalities yield
\begin{equation*}
\|\cL(\Delta)\|^2 \geq \frac{1}{10T} - c_2\frac{rd(d+T)}{m^2T} \max_{t} \|\Delta^{t)}\|^2\ln\pr{\frac{dT}{m}}
\end{equation*}
for universal positive constants $c_1, c_2$. Moreover, this inequality holds uniformly over all $\Delta \in\cC$ of norm $1$, as soon as the \cref{eq:rank,eq:operatornorm1} simultaneously hold. \cref{lemma:operatornorm,lemma:rscrank} allow to conclude.

\subsection{Proof of \cref{lemma:noise-level}}\label{proof:noise-level}
By trace duality property we have
\begin{align*}
\abs[\bigg]{\frac1{mT}\sum_{(i,t)\in[m]\times[T]}
\eps_i^t\inner{x_i^t, M\pj[t]}}
=\abs[\bigg]{\frac1{mT}\sum_{(i,t)\in[m]\times[T]}
\eps_i^t\tr(e_t (x_i^t)\tp M)}
\leq \norm{M}_*\norm{F},
\end{align*}
where $(e_t)_j=\indicator{j=t}$ for $1\leq j\leq T$ and
$F\coloneqq\frac1{mT}\sum_{t=1}^N\sum_{i=1}^m\eps_i^t e_t (x_i^t)\tp$.

The remaining of the proof aims at bounding the spectral norm of $F \in \R^{T\times d}$. Equivalently, we consider $A = \frac{T \sqrt{m}}{\sigma} F$. Note that the $t$-th row of $A$ is
\begin{equation*}
A_t = \frac1{\sqrt{m}\sigma} \sum_{i=1}^m \eps_i^t x_i^t.
\end{equation*}
In particular, the rows of $A_t$ are i.i.d. and isotropic, i.e. $\mathbb{E}[A_i\tp A_i] = I_d$. We can then show \cref{lemma:heavyconcentration} in \cref{sec:lemmastau}, which shows for a fixed $y \in \mathbb{S}^{d-1}$ and for any $\varepsilon \geq \eps_0$:
\begin{equation}\label{eq:heavyconcentration}
\mathbb{P}\pr{\|Ay\|_2^2 \geq (1+\eps)T} \leq (2T+c)\exp\pr{-c'\min\lbrace T\varepsilon, \sqrt{mT\varepsilon} \rbrace},
\end{equation}
where $c, c'$ and $\eps_0$ are universal constants.

\medskip

Let $\mathcal{N}$ be a $\frac{1}{4}$-net covering of $\mathbb{S}^{d-1}$ of cardinality $9^d$. Using \cref{eq:heavyconcentration} and a union bound argument leads, for any $\varepsilon \geq \varepsilon_0$ to
\begin{align*}
\mathbb{P}\pr{\max_{y\in \mathcal{N}} \|Ay\|_2^2 \geq (1+\varepsilon)T} & \leq (2T+c)\exp\pr{d\ln(9)-c'\min\lbrace T\varepsilon, \sqrt{mT\varepsilon} \rbrace}
\end{align*} 
As we assumed that $T = \Omega(d)$, taking $ \eps = c_2\pr{1+\frac{d^2}{mT}}$ for some positive constant $c_2$ gives
\begin{equation*}
\mathbb{P}\pr{\max_{y\in \mathcal{N}} \|Ay\|_2^2 \geq c_1(T + \frac{d^2}{m})}\leq (2T+c) e^{-c_0 d}.
\end{equation*}
Thus, with probability at least $1-(2T+c) e^{-c_0 d}$, $\max_{y\in \mathcal{N}} \|Ay\|_2 \leq \sqrt{c_1}\sqrt{T + \frac{d^2}{m}}$. Using a classical covering argument \citep[e.g.][Lemma 4.4.1]{vershynin2018high}, this implies that 
\begin{equation*}
\|A\|_2 \leq \frac{4}{3}\sqrt{c_1}\sqrt{T + \frac{d^2}{m}},
\end{equation*}
which leads to \cref{lemma:noise-level}, since $\|A\|_2 = \frac{T\sqrt{m}}{\sigma}\|F \|_2$.
%

\subsection{Auxiliary lemmas}\label{sec:lemmastau}

\begin{lemm}\label{lemma:pointwisersc}
For any matrix $M \in \mathbb{R}^{d\times T}$ and $\delta>0$, with probability at least $1-\delta$:
\begin{equation*}
\left| \|\cL(M)\|^2-\frac{\|M\|_F^2}{T} \right| \leq \frac{c}{4}\frac{\max_{t} \|M^{(t)}\|_2 \|M\|_F}{\sqrt{m}T}\sqrt{\log\frac{1}{\delta}} + c^2 \frac{\max_{t}\|M^{(t)}\|^2}{mT}\log\frac{1}{\delta},
\end{equation*}
where $c$ is a universal constant.
\end{lemm}

\begin{proof}
Recall that 
\begin{equation*}
\|\cL(M)\|^2 = \frac{1}{mT}\sum_{(i,t)\in[m]\times[T]} \langle M^{(t)}, x_i^t \rangle^2.
\end{equation*}
As a consequence, $\E[\|\cL(M)\|^2] = \frac{\|M\|_F^2}{T}$. Moreover, the random variables $\langle M^{(t)}, x_i^t \rangle^2$ are independent and $c'\|M^{(t)}\|^2$-sub-exponential for some constant $c'$. Bernstein inequality then yields for some universal constant $c_1$ \citep[][Theorem 2.8.1]{vershynin2018high}:
\begin{align*}
\p\pr{\left|\|\cL(M)\|^2 -\frac{\|M\|_F^2}{T} \right| \geq \frac{\eps}{mT}} & \leq 2\exp\pr{-c_1\min\left\lbrace \frac{\eps^2}{c'^2 m \sum_{t=1}^T \|M^{(t)}\|^4}, \frac{\eps}{c'\max_{t} \|M^{(t)}\|^2} \right\rbrace} \\
& \leq 2\exp\pr{-c_1\min\left\lbrace \frac{\eps^2}{c'^2 m \|M\|_F^2 \max_{t} \|M^{(t)}\|^2}, \frac{\eps}{c'\max_{t} \|M^{(t)}\|^2} \right\rbrace}.
\end{align*}
The second inequality comes from the inequality $\sum_{t=1}^T \|M^{(t)}\|^4 \leq \|M\|_F^2 \max_{t} \|M^{(t)}\|^2$.

\medskip

Taking $\eps = c\max_{t} \|M^{(t)}\|_2 \|M\|_F\sqrt{m\log\frac{1}{\delta}} + c^2 \max_{t}\|M^{(t)}\|^2\log\frac{1}{\delta}$ for a large enough constant $c$ leads to \cref{lemma:pointwisersc}.
\end{proof}

\begin{lemm}\label{lemma:heavyconcentration}
Let $A = \frac1{\sigma\sqrt{m}}\sum_{t=1}^N\sum_{i=1}^m\eps_i^t e_t (x_i^t)\tp$. For a fixed $y \in \mathbb{S}^{d-1}$ and any $\varepsilon\geq \eps_0$:
\begin{equation*}
\mathbb{P}\pr{\|Ay\|_2^2 \geq (1+\eps)T} \leq (2T+2)\exp\pr{-c'\min\lbrace T\varepsilon, \sqrt{mT\varepsilon} \rbrace},
\end{equation*}
where $\eps_0$ and $c'$ are positive universal constants.
\end{lemm}

\begin{proof}
By definition of $A$ and $A_t$:
\begin{align*}
\|Ay \|_2^2 & = \sum_{t=1}^T \langle A_t, y\rangle^2 \\
& = \sum_{t=1}^T \frac{1}{m}\pr{\sum_{i=1}^m \frac{\eps_i^t}{\sigma} \langle x_i^t, y \rangle}^2.
\end{align*}
Define for the remaining of the proof $X_i^t = \langle x_i^t, y \rangle$, $Y_i^t = \frac{\eps_i^t}{\sigma}$ and $Z_t = \frac1{m}\pr{\sum_{i=1}^m X_i^t Y_i^t}^2$.
By design of the problem, $X_i^t$ and $Y_i^t$ are independent, $0$-mean, $1$-subgaussian variables. The Bernstein inequality \citep[e.g.][Theorem 2.8.1.]{vershynin2018high} then gives for a universal positive constant $c'$:
\begin{align*}
\p\pr{\left| \sum_{i=1}^m X_i^t Y_i^t \right|\geq \eps} \leq 2 \exp\pr{-c\min\lbrace\frac{\varepsilon^2}{m}, \varepsilon\rbrace}.
\end{align*}
And so
\begin{align*}
\p\pr{Z_t\geq \eps} & = \p\pr{\frac{1}{m}\left( \sum_{i=1}^m X_i^t Y_i^t \right)^2\geq \eps} \\
& \leq 2 \exp\pr{-c\min\lbrace\varepsilon, \sqrt{m\varepsilon}\rbrace}.
\end{align*}
Moreover, note that the $Z_t$ are independent, positive and $\E[Z_t]=1$. We now want to use the following concentration result on heavy tailed distributions:
\begin{lemm}[\citealt{bakhshizadeh2020sharp}, Theorem 1]\label{lemma:heavycite}
Let $Z^{L} = Z \indicator{Z \leq L}$ and $c$ be any constant such that $\alpha_L\leq c$ for any $L \in \R$ where
\begin{equation*}
\alpha_L = \E\left[\pr{Z^L -  1}\indicator{Z\leq 1} + \pr{Z^L-1}^2\exp\left\lbrace \frac{c'\min(\eps, \sqrt{m\eps})-\ln(2)}{2L}(Z^L-1)\right\rbrace\indicator{Z^L> 1}\right]
\end{equation*}
 and 
\begin{equation*}
\eps_{\max} = \sup\left\lbrace \eps\geq 0 \mid \eps \leq \frac{c}{2} \frac{c'\min(\eps, \sqrt{m\eps})-\ln(2)}{T\eps} \right\rbrace,
\end{equation*}
then for any $\eps\geq\eps_{\max}$:
\begin{equation*}
\p\pr{\sum_{t=1}^T Z_t - T > T\eps} \leq 2\exp\pr{-\frac{c'\min(T\eps, \sqrt{Tm\eps})}{4}} + 2T \exp\pr{-c'\min(T\eps, \sqrt{Tm\eps})}.
\end{equation*}
\end{lemm}
Using Lemma~2 from \citep{bakhshizadeh2020sharp}, we can bound $\alpha_L$ as follows:
\begin{align*}
\alpha_L \leq 1 + 2\int_{0}^{\infty} \exp\pr{-\frac{c'}{2}\min(u+1, \sqrt{m(u+1)})}(2u + \frac{c't}{2}\min(u, \sqrt{mu)}))\df u.
\end{align*}
A simple calculation allows to bound the integral by some constant value that does not depend on $m$, i.e. we can use \cref{lemma:heavycite} where $c$ is a universal constant. $\eps_{\max}$ is then smaller than some universal constant $\eps_0$ and thus, for $\eps \geq \eps_0$:
\begin{align*}
\p\pr{\sum_{t=1}^T Z_t > (1+\eps)T}\leq (2T+2)\exp\pr{-\frac{c'}{4}\min\{T\eps, \sqrt{mT\eps}\}}.
\end{align*}
This leads to \cref{lemma:heavyconcentration} as $\sum_{t=1}^T Z_t = \| Ay\|_2^2$.
\end{proof}

\subsection{Proof of \cref{coro:angle}}\label{sec:angleproof}
As $M^*$ is of rank $r$, the definition of $\tilde{M}$ in \cref{eq:projrank} implies
\begin{equation*}
    \|\hat{M} - \tilde{M}\|_F \leq \|\hat{M} - M^*\|_F.
\end{equation*}
By triangle inequality, $\|\tilde{M}- M^*\|_F \leq 2 \|\hat{M} - M^*\|_F$.
A direct application of Lemma~16 by \citet{tripuraneni2021provable}, which we recall below, then allows to conclude.
\begin{lemm}[\citealt{tripuraneni2021provable}, Lemma~16]
Suppose we have matrices $\tilde{B},B \in \R^{d\times r}$ and $\tilde{\alpha},\alpha\in\R^{r\times T}$ such that $\tilde{B}^\top \tilde{B} = I_r = B^{\top}B$ and $\|\tilde{B}\tilde{\alpha}-B \alpha\|^2_F\leq \varepsilon$,
then $$\sin^2 \theta\left(B,\tilde{B}\right) \leq \frac{\varepsilon}{\lambda_r\left(\alpha\alpha^\top\right)}.$$
\end{lemm}

\subsection{Proof of \cref{thm:MoM}}\label{sec:MoMproof}

Similarly to Theorem~7 by \citet{tripuraneni2021provable}, we can show the following lemma. Its proof is omitted as it follows the exact same lines as the proof from \citep{tripuraneni2021provable}.

\begin{lemm}\label{lemma:MoM1}
For any set $S \subset [T]$ such that $\card{S}\geq c$, then with probability at least $1-\frac{1}{m^2T^2}$:
\begin{equation*}
\left\| \frac{1}{m\, \card{S}}\sum_{i=1}^m \sum_{t\in S} \pr{(y_i^t)^2 x_i^t \pr{x_i^t}\tp - \mathbb{E}\left[ (y_i^t)^2 x_i^t \pr{x_i^t}\tp\right]}\right\| \leq c' \log^3(dmT) \pr{\sigma^2 + C} \pr{\sqrt{\frac{d}{mT}} + \frac{d}{mT}},
\end{equation*}
for universal constants $c$ and $c'$.
\end{lemm}
\cref{lemma:expectation} below explicits the expected term in \cref{lemma:MoM1}. It generalizes the Lemma~2 by \citet{tripuraneni2021provable} to the case of any spherically symmetric distribution.

\begin{lemm}\label{lemma:expectation}
If $x$ is spherically symmetric and $y = \langle u, x \rangle + \eps$, where $\eps$ is independent from $x$, centered and $\sigma$-subgaussian, then
\begin{equation*}
\mathbb{E}\left[ (y^t)^2 x^t \pr{x^t}\tp\right] = \kappa u u\tp + \pr{\var(\eps) + \gamma \|u\|^2} I_d, 
\end{equation*}
where $\gamma = \E[\langle e_1, x\rangle^2\langle e_2, x\rangle^2]$ and $\kappa = \E[\langle e_1, x\rangle^4] - \gamma$.
\end{lemm}
Note that in the Gaussian case $\gamma = 1$ and $\kappa = 2$, hence recovering the previous result by \citet{tripuraneni2021provable}.

\begin{proof}
Using the relation between $y$ and $x$, it comes:
\begin{equation*}
    \mathbb{E}\left[ (y^t)^2 x^t \pr{x^t}\tp\right] = \var(\eps) I_d + \E[x\tp u u\tp x x x\tp].
\end{equation*}
By rescaling and invariance under orthogonal transformations, it actually suffices to show \cref{lemma:expectation}. We then have, still by invariance under orthogonal transformation
\begin{align*}
    \pr{\E[x\tp e_1 e_1\tp x x x\tp]}_{i,j} &= \E[\langle x, e_1 \rangle^2 \langle x, e_i\rangle \langle x, e_j \rangle]\\
    & = \begin{cases}\E[\langle x, e_1 \rangle^4 ] \quad \text{if } i=j=1, \\
    \E[\langle x, e_1 \rangle^2 \langle x, e_2 \rangle^2] \quad \text{if } i=j\neq1, \\
    \E[\langle x, e_1 \rangle^2 \langle x, e_2 \rangle \langle x, e_3 \rangle] \quad \text{if } i \neq j \text{ and } i\neq 1 \text{ and } i\neq j, \\
    \E[\langle x, e_1 \rangle^3 \langle x, e_2 \rangle ] \quad \text{if } 1=i \neq j \text{ or } 1=j \neq i.
    \end{cases}
\end{align*}
By considering the orthogonal transformation $e_2 \mapsto -e_2$, the term is $0$ in the last two cases. This finally gives
\begin{equation*}
    \E[x\tp e_1 e_1\tp x x x\tp] = \kappa e_1 e_1\tp + \pr{\var(\eps) + \gamma \|e_1\|^2} I_d,
\end{equation*}
which leads to \cref{lemma:expectation} by rescaling and rotating $e_1$ to $\frac{u}{\|u\|}$.
\end{proof}

We can now bound the subspace estimation error due to the estimators $\hat{B}_1$ and $\hat{B}_2$ in \cref{algo:MoM}.

\begin{lemm}\label{lemma:angleMoM}
Consider $\hat{B}_1$, $\hat{B}_2$ defined in \cref{algo:MoM} and $B$ defined in \cref{sec:meta}. 
Then if $mT \geq c \frac{\log^6(dmT)\pr{\sigma^4+C^2}r^2}{\kappa^2 \nu^2}d$, with probability at least $1-\frac{2}{m^2T^2}$:
\begin{equation*}
    \sin\theta\pr{\hat{B}_i,B} \leq c'\log^3(dmT)r\frac{\sigma^2+C}{\kappa \bar{\nu}} \sqrt{\frac{d}{mT}}, \quad \text{ for } i=1,2,
\end{equation*}
where $\theta(B_i, B)$ is the principal angle between the subspaces induced by $B_i$ and $B$, $c$ and $c'$ are universal constants.
\end{lemm}
\begin{proof}
The proof assumes that the concentration bound given by \cref{lemma:MoM1} holds and uses the Davis-Kahan theorem on it.
Note that $\hat{B}_1$ derives from the top-$r$ SVD of $A = \frac{1}{m\lfloor \frac{T}{2}\rfloor}\sum_{i=1}^m \sum_{t=\lceil\frac{T}{2}\rceil +1} (y_i^t)^2 x_i^t \pr{x_i^t}\tp$.

Thanks to \cref{lemma:MoM1,lemma:expectation}, it holds with probability at least $1-\frac{1}{m^2T^2}$ that
\begin{gather*}
    A = \kappa \bar{\Gamma} + \pr{\var(\eps)+\gamma \mathrm{tr}(\bar{\Gamma})I_d} + E, \\
    \text{ with } \|E\| \leq c'\log^3(dmT)(\sigma^2 +C)\pr{\sqrt{\frac{d}{mT}} + \frac{d}{mT}}.
\end{gather*}
For $T$ large enough, as given in \cref{lemma:angleMoM}, we have $\|E\|\leq \frac{\kappa \bar{\nu}}{2r}$. Now note that $\lambda_r(A-E) - \lambda_{r+1}(A-E) \geq \kappa \frac{\bar{\nu}}{r}$ and $B$ corresponds to the top-$r$ left singular vectors of the $A-E$.

Davis-Kahan theorem then yields that $\sin\theta\pr{\hat{B}_1,B}\leq \frac{2r\|E\|}{\kappa\bar{\nu}}$.
The same argument for $\hat{B}_2$ finally yields to \cref{lemma:angleMoM}.
\end{proof}

\cref{thm:MoM} then follows from \cref{lemma:angleMoM} and Theorem~4 from \citet{tripuraneni2021provable}, using the independence between the features used for the estimator $\hat{B}_i$ and the ones used for the estimator~$\hat{\alpha}_i$.

\medskip

It now just remains to show that $\kappa \geq 0$. We actually have the following series of (in)equalities:
\begin{align*}
    \E[\langle x, e_1\rangle^2 \langle x, e_2\rangle^2] & \leq \E[\langle x, e_1\rangle^2] \E[\langle x, e_2\rangle^2] \\
    & = \E[\langle x, e_1\rangle^2]^2 \\
    & \leq \E[\langle x, e_1\rangle^4].
\end{align*}
The first inequality derives from Cauchy-Schwarz inequality. The equality is by invariance over orthogonal transformation, while the last inequality comes from Jensen inequality. This directly implies that $\kappa \geq 0$.

\medskip

Now show that the case of equality is impossible. In particular, it would imply from the first inequality that $\langle x, e_1\rangle^2 = \langle x, e_2\rangle^2$ almost surely. By invariance under rotation, we can even show that for any $u$ of norm $1$: $\langle x, e_1\rangle^2 = \langle x, u\rangle^2$, which implies that $x=0$ almost surely, hence leading to a contradiction. We thus have $\kappa>0$.

\end{document}